\documentclass[12pt,letterpaper]{article}

\usepackage{graphicx}
\usepackage{amsmath}
\usepackage{amsthm}
\usepackage{amssymb}
\usepackage{mathtools}
\usepackage{booktabs}
\usepackage{subcaption, geometry}

\newtheorem{theorem}{Theorem}
\newtheorem{proposition}{Proposition}

\newtheorem{fact}{Fact}
\usepackage[pagebackref,breaklinks,colorlinks]{hyperref}

\usepackage[capitalize]{cleveref}
\crefname{section}{Sec.}{Secs.}
\Crefname{section}{Section}{Sections}
\Crefname{table}{Table}{Tables}
\crefname{table}{Tab.}{Tabs.}

\begin{document}

\title{Nested Hyperbolic Spaces for Dimensionality Reduction and Hyperbolic NN Design}

 \author{Xiran Fan\\
 University of Florida\\
Department of Statistics\\
 {\tt\small fanxiran@ufl.edu}
 \and
 Chun-Hao Yang\\
 National Taiwan University\\
Institute of Applied Mathematical Science\\
 {\tt\small chunhaoy@ntu.edu.tw}
  \and
 Baba C.\ Vemuri\\
 Department of CISE, University of Florida\\
 {\tt\small vemuri@ufl.edu}
 }
\maketitle

\begin{abstract}
Hyperbolic neural networks have been popular in the recent past due to their ability to represent hierarchical data sets effectively and efficiently. The challenge in developing these networks lies in the nonlinearity of the embedding space namely, the Hyperbolic space. Hyperbolic space is a homogeneous Riemannian manifold of the Lorentz group which is a semi-Riemannian manifold, i.e.\ a manifold equipped with an indefinite metric. Most existing methods (with some exceptions) use local linearization to define a variety of operations paralleling those used in traditional deep neural networks in Euclidean spaces. In this paper, we present a novel fully hyperbolic neural network which uses the concept of projections (embeddings) followed by an intrinsic aggregation and a nonlinearity all within the hyperbolic space. The novelty here lies in the projection which is designed to project data on to a lower-dimensional embedded hyperbolic space and hence leads to a nested hyperbolic space representation independently useful for dimensionality reduction. The main theoretical contribution is that the proposed embedding is proved to be isometric and equivariant under the Lorentz transformations, which are the natural isometric transformations in hyperbolic spaces. This projection is computationally efficient since it can be expressed by simple linear operations, and, due to the aforementioned equivariance property, it allows for weight sharing. The nested hyperbolic space representation is the core component of our network and therefore, we first compare this ensuing nested hyperbolic space representation -- independent of the network -- with other dimensionality reduction methods such as tangent PCA, principal geodesic analysis (PGA) and HoroPCA. Based on this equivariant embedding, we develop a novel fully hyperbolic graph convolutional neural network architecture to learn the parameters of the projection. Finally, we present experiments demonstrating comparative performance of our network on several publicly available data sets.
\end{abstract}
\newpage
\section{Introduction}\label{sec:intro}
Hyperbolic geometry is a centuries old field of non-Euclidean geometry and has recently found its way into the field of machine learning, in particular into deep learning in the form of hyperbolic neural networks (HNNs) or hyperbolic graph convolutional networks (HGCNs) and recently for dimensionality reduction of data embedded in the hyperbolic space. In this paper, we will discuss both problems namely, dimensionality reduction in hyperbolic spaces and HNN architectures.  In particular, we will present novel techniques for both these problems. In the following, we present literature review of the two above stated problems and establish the motivation for our work. A word on terminology, we will use the term hyperbolic neural network and hyperbolic graph (convolutional) neural network synonymously for the rest of the paper.

\begin{figure}[t]
  \centering
   \includegraphics[width=0.4\linewidth]{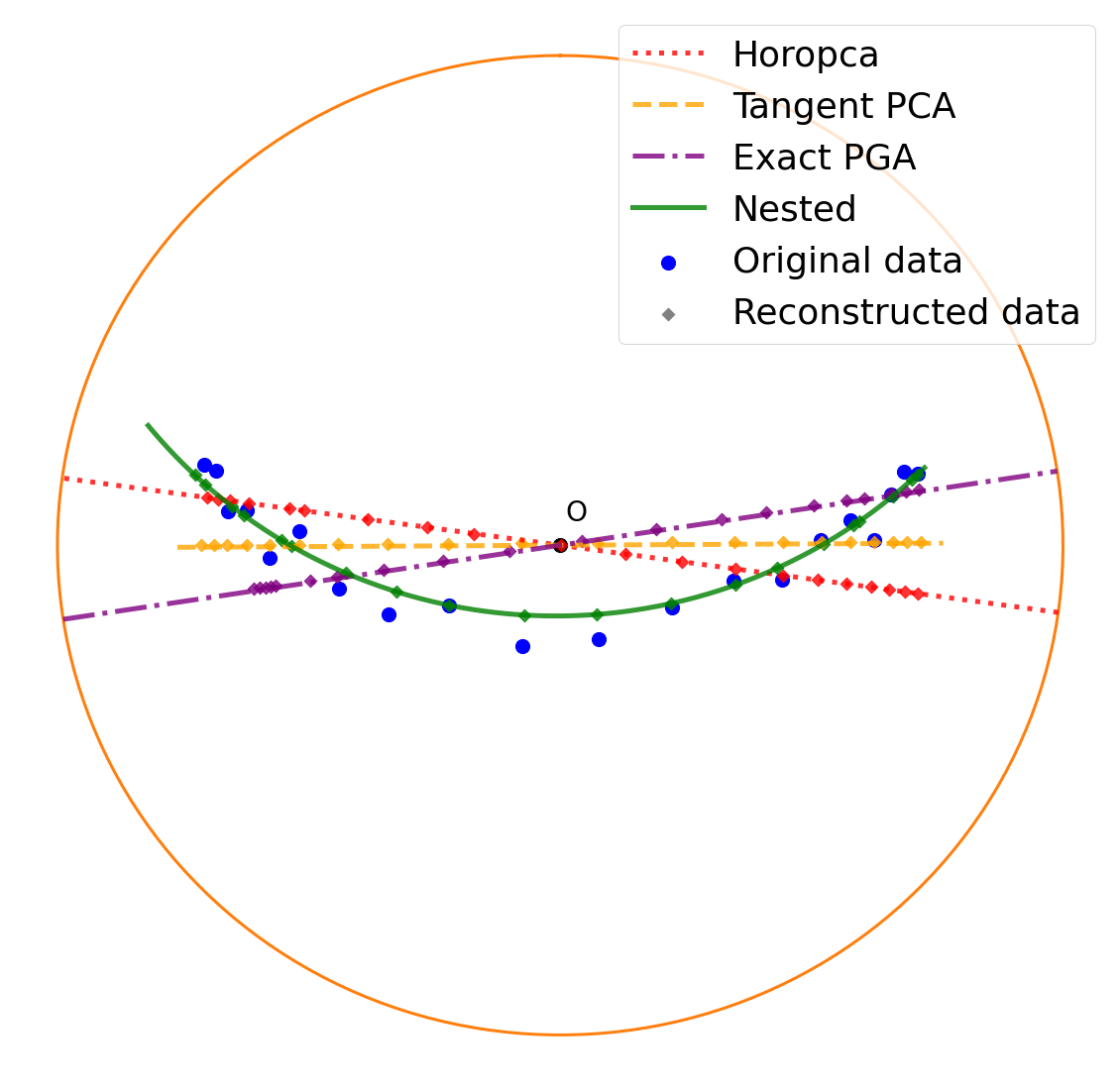}
   \caption{Projections of data from a 2-dimensional hyperbolic space to a 1-dimensional hyperbolic space using different dimensionality reduction methods. The results are visualized in the Poincar\'{e} disk. Original data (blue dots) lie in a 2-dimensional hyperbolic space and have a zero mean (origin of the Poincar\'{e} disk). The HoroPCA direction (red dotted line) and the principal geodesic obtained by tangent PCA (orange dashed line) and Exact PGA (purple dash-dotted line) fail to capture the main trend of the data since they are restricted to learn a geodesic submanifold passing through the mean. In contrast, our nested hyperbolic (NH) representation (green solid line), captures the data trend more accurately. The diamond markers on each line represent the reconstructed data from each method. The reconstruction errors for HoroPCA, tangent PCA, EPGA and the proposed NH scheme in this example are, 0.1708, 0.1202, 0.1638 and 0.0062 respectively.}
   \label{fig:mean0example}
\end{figure}

\subsection{Dimensionality Reduction of Manifold-valued Data} 

Dimensionality reduction is a fundamental problem in machine learning with applications in computer vision and many other fields of engineering and sciences. The simplest and most popular method among these is the principal component analysis (PCA), which was proposed more than a century ago (see \cite{jolliffe2016principal} for a review and some recent developments on PCA). PCA however is limited to data in vector spaces. For data that are manifold-valued, principal geodesic analysis (PGA) was presented in \cite{fletcher2004principal}, which yields the projection of data onto principal geodesic submanifolds passing through an intrinsic (Fr\'{e}chet) mean \cite{Frechet48} of the data. They find the geodesic submanifold of a lower dimension that maximizes the projected variance and computationally, this was achieved via linear approximation, i.e., applying PCA on the tangent space anchored at the Fr\'{e}chet mean. This is sometimes referred to as the tangent PCA (tPCA). This approximation however requires the data to be clustered around the Fr\'{e}chet mean, otherwise the tangent space approximation to the manifold leads to inaccuracies. Subsequently, \cite{sommer2010manifold} presented the Exact PGA (EPGA) algorithm, which does not use any linear approximation. However, EPGA is computationally expensive as it requires two non-linear optimizations steps per iteration (projection to the geodesic submanifold and finding the new geodesic direction such that the reconstruction error is minimized). Later, authors in \cite{chakraborty2016efficient} developed a version of EPGA for constant sectional curvature manifolds, namely the hypersphere and the hyperbolic space, by deriving closed form formulae for the projection. There are many variants of PGA and we refer the reader to \cite{banerjee2017sparse,huckemann2010intrinsic,zhang2013probabilistic} for the details. More recently, Barycentric subspace analysis (BSA)  was proposed in \cite{pennec2018barycentric} which finds a more general parameterization  of a nested sequence of submanifolds via the minimization of unexplained variance. Another useful dimensionality reduction scheme is the Principal curves \cite{hastie1989principal} and their generalization to Riemannian manifolds \cite{hauberg2015principal} that are more appropriate for certain applications.
%Other variants of PGA include but not limited to sparse exact PGA \cite{banerjee2017sparse}, geodesic PCA \cite{huckemann2010intrinsic}, and probabilistic PGA \cite{zhang2013probabilistic}. All the above methods focus on projecting data to a \emph{geodesic submanifold} as in PCA where one projects data to a vector subspace. Instead, one can also project data to a submanifold that minimizes the reconstruction error without any further restrictions, e.g.\ being geodesic. This is the generalization of the principal curve \cite{hastie1989principal} to Riemannian manifolds presented in \cite{hauberg2016principal}. 

A salient feature of PCA is that it yields nested linear subspaces, i.e., the reduced dimensional principal subspaces form a nested hierarchy. This idea was exploited in \cite{jung2012analysis} where authors proposed the \emph{principal nested spheres} (PNS) by embedding an $(n-1)$-sphere in to an $n$-sphere, the embedding however is not necessarily isometric. Hence, PNS is more general than PGA in that PNS does not have to be geodesic. Similarly, for the manifold $P_n$ of $(n \times n)$ symmetric positive definite (SPD) matrices, authors in \cite{harandi2018dimensionality} proposed a geometry-aware dimensionality reduction by projecting data on $P_n$ to $P_m$ for some $m \ll n$. More recently, the idea of constructing a nested sequence of manifolds was presented in \cite{yang2021nested} where authors unified and generalized the nesting concept to general Riemannian homogeneous manifolds, which form a large class of Riemannian manifolds, including the hypersphere, $P_n$, the Grassmannian, Stiefel manifold, Lie groups, and others. Although the general framework in \cite{yang2021nested} seems straightforward and applicable to hyperbolic spaces, many significantly important  technical aspects need to be addressed and derived in detail. In this paper, we will present novel derivations suited for the hyperbolic spaces -- a projection operator which is proved to yield an isometric embedding, and a proof of equivariance to isometries of the projection operator --  which will facilitate the construction of nested hyperbolic spaces and the hyperbolic neural network. Note that there are five models of the hyperbolic space namely, the hyperboloid (Lorentz) model, the Poincar\'{e} disk/ball model, the Poincar\'{e} half plane model, the Klein model and the Jemisphere model \cite{JamesCannon1997}. All these models are isometrically equivalent but some are better suited than others depending on the application. We choose the Lorentz model of the hyperbolic space with a Lorentzian metric in our work. The choice of this model and the associated metric over other models is motivated by the properties of Riemannian optimization efficiency and numerical stability afforded \cite{nickel2018learning,chen2021fully}.

%A nested sequence of relations that determine a nested sequence of submanifolds which are not necessarily geodesic were presented in \cite{damon2014backwards}. They showed various examples, including Euclidean space and the $n$-sphere, depicting how the nested relations generalized PCA and PNS. It is however unclear how to achieve this for an arbitrary Riemannian manifold. More recently, another generalization of PGA was proposed by \cite{pennec2018barycentric}, called the exponential barycentric subspace (EBS). A $k$-dimensional EBS is defined as the locus of weighted exponential barycenters of $(k + 1)$ affinely independent reference points. The EBSs are naturally nested by removing or adding reference points. None of the above methods are directly applicable to the Hyperbolic space excepting the PGA, EPGA and EBS. 

Most recently, an elegant approach called \emph{HoroPCA} was proposed in \cite{chami2021horopca}, for dimensionality reduction in hyperbolic spaces. In particular, the authors represented the hyperbolic space using the Poincar\'{e} model and they proposed to generalize the notion of direction and the coordinates in a given direction using \emph{ideal points} (points at infinity) and the \emph{Busemann coordinates} (defined using the \emph{Busemann function}) \cite{busemann1955geometry}. The levels sets of the Busemann function, called the \emph{horospheres}, resemble the hyperplanes (or affine subspaces) in Euclidean spaces and hence the dimensionality reduction is achieved by a projection that moves points along a horosphere. The data is then projected to a \emph{geodesic hull of a base point $b$ and a number of ideal points $p_1,\ldots, p_K$}, {\it which is also a geodesic submanifold}. \emph{This is the key difference between HoroPCA and our proposed method which leads to a significant difference in performance.} This is evident from the toy example in Figure~\ref{fig:mean0example} which depicts the reduced dimensional representations obtained by our method in comparison to those from EPGA, HoroPCA, and tangent PCA. Note that all of the other methods yield submanifold representations that do not capture the data trend accurately, unlike ours. More comprehensive comparisons will be made in a later section.

To briefly summarize, our first goal in this paper is to present a nested hyperbolic space representation for dimensionality reduction 
%following the approach proposed by \cite{yang2021nested}, 
and we will demonstrate, via synthetic examples and real datasets, that it achieves a lower reconstruction error in comparison to other competing methods. 

\subsection{Hyperbolic Neural Networks}

Several researchers have demonstrated that the hyperbolic space is apt for modeling hierarchically organized data, for example, graphs and trees \cite{sarkar2011low,sala2018representation,nickel2017poincare}. Recently, the formalism of Gyrovector spaces (an algebraic structure) \cite{ungar2005gyrovector} was applied to the hyperbolic space to define basic operations paralleling those in vector spaces and were used to build a hyperbolic neural network (HNN) \cite{ganea2019hyperbolic,shimizu2020hyperbolic}. The Gyrovector space formalism facilitates performing M\"{o}bius additions and subtractions in the Poincare model of the hyperbolic space.
% \textcolor{red}{Xiran, i tried modifying the following sentences but i really don't understand what you are trying to convey and hence couldn't do so properly. I modified it by guessing what you wanted to say and I am not sure.  This needs checking. Don't forget.}
HNNs have been successfully applied to word embeddings \cite{tifrea2018poincare} as well as image embeddings \cite{khrulkov2020hyperbolic}. Additionally, several existing deep network architectures have been modified to suit hyperbolic embeddings of data, e.g.,
graph networks \cite{liu2019hyperbolic,chami2019hyperbolic}, attention module \cite{gulcehre2018hyperbolic}, and variational auto-encoders \cite{mathieu2019charline,park2021unsupervised}. These hyperbolic networks were shown to perform comparably or even better than their euclidean counterparts.
%have been adopted hyperbolic structure, such as graph networks \cite{liu2019hyperbolic,chami2019hyperbolic}, attention module \cite{gulcehre2018hyperbolic}, and variational auto-encoders \cite{mathieu2019charline}. These hyperbolic networks are shown to perform comparably or even better than euclidean deep networks.

Existing HNNs have achieved moderate to great successes in multiple areas and shown great potential in solving complex problems. However, most of them use tangent space approximations to facilitate the use of vector space operations prevalent in existing neural network architectures. There are however some exceptions, for instance, the authors in \cite{dai2021hyperbolic} developed what they call a Hyperbolic-to-Hyperbolic network and the authors in \cite{chen2021fully} also developed a fully Hyperbolic network. They both considered the use of Lorentz transformations on hyperbolic features since the Lorentz transformation matrix acts transitively on a hyperbolic space and thus preserves the global hyperbolic structure. Each Lorentz transformation is a composition of a Lorentz rotation and a rotation free Lorentz transformation called the Lorentz boost operation. Authors in \cite{dai2021hyperbolic} only use Lorentz rotation for hyperbolic feature transformations while authors in \cite{chen2021fully} build a fully-connected layer in hyperbolic space (called a hyperbolic linear layer) parameterized by an arbitrary weight matrix (not necessarily invertible) which is applied to each data point in the hyperbolic space resulting in a mapping from a hyperbolic space to itself. This procedure is ad hoc in the sense that it does not use the intrinsic characterization of the hyperbolic space as a homogeneous space with the isometry group being the Lorentz group.

Lorentz transformations are however inappropriate for defining projection operations (required for reducing the dimensionality) as they preserve the Lorentz model only \emph{when there is no change in dimension}. In other words, to find a lower-dimensional hyperbolic space representation for data embedded in a higher-dimensional hyperbolic space, one cannot use Lorentz transformations directly. Hence, we propose to use an isometric embedding operation mentioned in the previous subsection as the building block to design a hyperbolic neural network. We will now briefly summarize our proposed model and the contributions of our work. 

\subsection{Proposed Model and Contributions}
%{\color{red} Now you start describing the proposed work briefly and then list out the key contributions of the paper. Followed by a paragraph describing the organization of the paper.}

Inspired by \cite{jung2012analysis} and \cite{yang2021nested}, we construct a nested representation in a hyperbolic space to extract the hyperbolic features. Such a nested (hierarchical) hyperbolic space representation has the advantage that the data in reduced dimensions remains in a hyperbolic space.  Hereafter, we refer to these nested hyperbolic spaces as nested hyperboloids (NHs).
%Authors in \cite{jung2012analysis} stated the fact that a low dimensional unit sphere can be embedded in a high dimensional unit sphere by using a nested representation scheme. Furthermore, authors in \cite{yang2021nested} recently showed that all the Riemannian homogeneous spaces enjoy such nested structures. Hyperbolic model, as a Riemannian homogeneous space, also admits such nested representation.The nested hyperbolic space representation first leads to a dimensionality reduction method in a hyperbolic space, hereafter referred to as Nested Hyperboloid (NH). 
As a dimensionality reduction method in Riemannian manifolds, the learned lower dimensional submanifold in NH is not required to pass the Fr\'{e}chet mean unlike in PGA and need not be a geodesic submanifold as in HoroPCA, PGA or EPGA. In the experiments section, we will demonstrate that this leads to much lower reconstruction error in comparison to the aforementioned dimensionality reduction methods.

After defining the projection which leads to an embedding within hyperbolic spaces of different dimensions, these projections/embeddings are used to define a feature transformation layer in the hyperbolic space. This layer is then composed with a hyperbolic neighborhood aggregation operation/layer and appropriate non-linear operations in between namely, the tangent-ReLU, to define a novel nested hyperbolic graph convolutional network (NHGCN) architecture. 

%Our model uses the equivalent of linear decision boundaries in hyperbolic spaces called horrocycles to achieve classification of data in the hyperbolic space.\textcolor{red}{Xiran, is what I said here about the classification correct?}
% {\color{red} Some details need to be filled in after the network experiments are done.}

The rest of the paper is organized as follows. In Section~\ref{sec:prelim}, we briefly review the geometry of hyperbolic space.
%and, in particular, we will discuss how the hyperbolic space can fit into the framework proposed by \cite{yang2021nested}. 
In Section~\ref{sec:mainwork}, we explicitly give the projection and embedding to map data between hyperbolic spaces of different dimensions. We also present a novel hyperbolic graph convolutional neural network architecture based on these projections and tangent-ReLU activation. In Section~\ref{sec:experiment}, we first present the performance of NH as a dimensionality reduction method and compare with other competing methods, including EPGA, tangent-PCA and HoroPCA. Next, we compare our NHGCN with other hyperbolic networks on the problems of link prediction and node classification on four graph datasets described and used  in \cite{chami2019hyperbolic}.
%{\color{red} XXX datasets (details regarding the dataset need to be filled in)}. 
Finally, we draw conclusions in Section~\ref{sec:conclusion}.

\section{Preliminaries}\label{sec:prelim}
In this section, we briefly review relevant concepts of hyperbolic geometry. In this paper, we will regard the hyperbolic space as a homogeneous Riemannian manifold of the Lorentz group and present a few important geometric concepts, including the geodesic distance and the exponential map, in the hyperbolic space, which are used in our work. The materials presented in this section can be found in most textbooks on hyperbolic spaces, for example \cite{ratcliffe2006foundations, cannon1997hyperbolic}.

\subsection{Lorentzian Space and Hyperbolic Space}\label{sec:lorentz_space}

\begin{figure} \centering
    \begin{subfigure}[b]{0.475\linewidth}
        \includegraphics[width=0.8\linewidth, trim=100 100 100 100,clip]{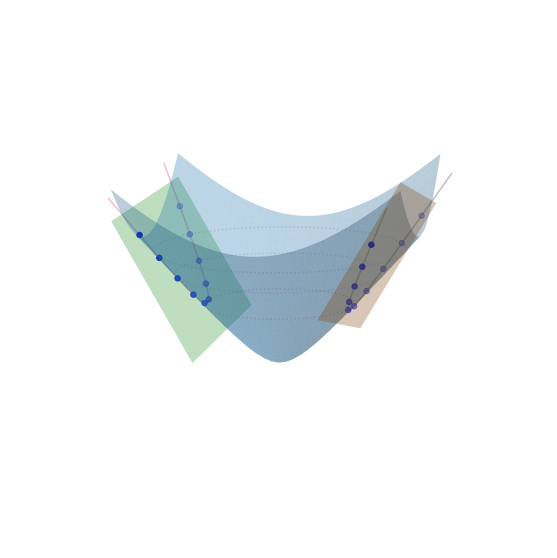}
        \caption{Lorentz rotation}\label{fig:Lorentz_rotation}
    \end{subfigure} 
    \begin{subfigure}[b]{0.475\linewidth}    
        \includegraphics[width=0.8\linewidth, trim=100 100 100 100,clip]{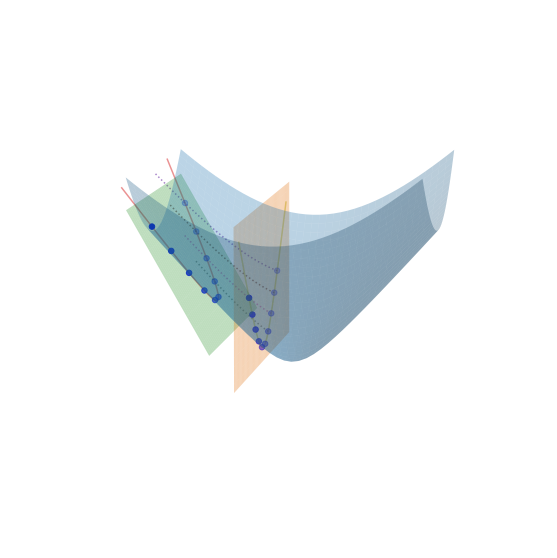}
        \caption{Lorentz boost}\label{fig:Lorentz_boost}    
    \end{subfigure} 
    \caption{Illustration of two kinds of Lorentz transformation, Lorentz rotation and Lorentz boost in a Lorentz model. They are isometric operations of the Lorentz model.}
    \label{fig:lorentz_transformation}
\end{figure}

As mentioned in Section \ref{sec:intro},  there are several (isometrically) equivalent models of a hyperbolic space, including the Poincar\'{e} model, Klein model, the upper-half space model, and the Jemisphere model \cite{JamesCannon1997}.  We choose to use the hyperboloid (Lorentz) model of the hyperbolic space in this paper due to its numerical stability property which is very useful for the optimization problem involved in the training and test phases. Our technique is however applicable to all of the models due to the isometric equivalence of the models.

The $(n+1)$-dimensional \emph{Lorentzian space} $\mathbb{R}^{1,n}$ is the Euclidean space $\mathbb{R}^{n+1}$ equipped with a bilinear form
\begin{equation*}\label{eq:bilinearform}
    \langle \boldsymbol{x},\boldsymbol{y}\rangle_{L} =-x_0y_0+x_1y_1 + \cdots +x_ny_n 
\end{equation*}
where $\boldsymbol{x} = [x_0,x_1,\ldots,x_{n}]^T, \boldsymbol{y} = [y_0,y_1,\ldots,y_{n}]^T \in \mathbb{R}^{n+1}$. This bilinear form is sometimes referred to as the Lorentzian inner product although it is not positive-definite. We denote the norm, called \emph{Lorentzian norm}, induced by the Lorentzian inner product by $\|\boldsymbol{x}\|_L = \sqrt{\langle \boldsymbol{x}, \boldsymbol{x} \rangle_L}$. Note that $\|\boldsymbol{x}\|_L$ is either positive, zero, or positive imaginary.

We consider the following submanifold of $\mathbb{R}^{1,n}$
\begin{equation*}\label{eq:hyperboloid}
    \mathbb{L}^n : = \{\boldsymbol{x} = [x_0,\ldots,x_{n}]^T\in \mathbb{R}^{n+1}: \|\boldsymbol{x}\|_L^2=-1,x_0 > 0\}
\end{equation*}
This is called the $n$-dimensional \emph{hyperboloid model} of one sheet of a hyperbolic space defined in $\mathbb{R}^{n+1}$.

% \textcolor{blue}{
% The isometric relationship between hyperboloid model and Poincar\'{e} model is given by:
% \begin{equation*}
%     x = [x_0,\ldots,x_{n}]^T\in \mathbb{H}^{n} \Leftrightarrow \left[ \frac{x_1}{1+x_0},\ldots \frac{x_n}{1+x_0}\right]\in \mathbb{B}^{n} 
% \end{equation*}
% The isometric relationship between hyperboloid model and Klein model is given by:
% \begin{equation*}
%     x = [x_0,\ldots,x_{n}]^T\in \mathbb{H}^{n} \Leftrightarrow \left[ \frac{x_1}{x_0},\ldots \frac{x_n}{x_0}\right]\in \mathbb{K}^{n} 
% \end{equation*}
% }
\subsection{Lorentz Transformations}
In the Lorentzian space, the linear isometries are called the \emph{Lorentz transformation}, i.e.\ the map $\phi:\mathbb{R}^{n+1} \to \mathbb{R}^{n+1}$ is a Lorentz transformation if $\langle \phi(x), \phi(y) \rangle_L = \langle x, y\rangle_L$ for any $x, y \in \mathbb{R}^{n+1}$. It is easy to see that all Lorentz transformations form a group under composition, and this group is denoted by $\mathbf{O}(1,n)$, called the \emph{Lorentz group}. The matrix representation of $\mathbf{O}(1,n)$ in $\mathbb{R}^{n+1}$ is defined as follows. Let $J_n = \text{diag}(-1,I_n)$ where $I_n$ is the $n\times n$ identity matrix and $\text{diag}(\cdot)$ denotes a diagonal matrix. Then, $\mathbf{O}(1,n)$ is defined as $\mathbf{O}(1,n) \coloneqq \{A\in M_{n+1}(\mathbb{R}): AJ_nA^T = A^TJ_nA = J_n\}$. There are a few important subgroups of $\mathbf{O}(1,n)$: (i) the subgroup $\mathbf{O}^+(1,n) \coloneqq \{A \in \mathbf{O}(1, n): a_{11} > 0\}$ is called the \emph{positive Lorentz group}; (ii) the subgroup $\mathbf{SO}(1,n) \coloneqq \{A \in \mathbf{O}(1,n): \det(A) = 1\}$ is called the \emph{special Lorentz group}; (iii) the subgroup $\mathbf{SO}^+(1,n) \coloneqq \{A \in \mathbf{SO}(1,n): a_{11} > 0\}$ is called the \emph{positive special Lorentz group}. Briefly speaking, the special Lorentz group preserves the orientation, and the positive Lorentz group preserves the sign of the first entry of $\boldsymbol{x} \in \mathbb{L}^n$.

\subsection{Riemannian Geometry of Hyperbolic Space}

A commonly used Riemannian metric for $\mathbb{L}^n \subset \mathbb{R}^{n+1}$ is the restriction of the Lorentz inner product to the tangent space of $\mathbb{L}^n$. \emph{Note that even though the Lorentz inner product is not positive-definite, when restricted to the tangent space of $\mathbb{L}^n$, it is positive-definite. Hence, $\mathbb{L}^n$ is a Riemannian manifold with constant negative sectional curvature}. Furthermore, the group of isometries  of $\mathbb{L}^n$ is precisely $\mathbf{O}^+(1,n)$ and the group of orientation-preserving isometries is $\mathbf{SO}^+(1,n)$. We now state a few useful facts about the group of isometries that are used in this paper and refer the interested reader to \cite{gallier2012notes} for details.

\begin{fact}
The positive special Lorentz group $\mathbf{SO}^+(1,n)$ acts transitively on $\mathbb{L}^n$ where the group action is defined as $x \mapsto Ax$ for $x \in \mathbb{L}^n$ and $A \in \mathbf{SO}^+(1,n)$.
\end{fact}

\begin{fact}
Let $\boldsymbol{x} = [1,0,\ldots,0]^T \in \mathbb{L}^n$. The isotropy subgroup $G_{\boldsymbol{x}}$ is given by
\begin{align*}\label{eq:ifotropy_subgroup}
G_x & \coloneqq \{A \in \mathbf{SO}^+(1,n): A\boldsymbol{x} = \boldsymbol{x}\}\\
& = \left\{\left[ \begin{array}{cc}
1 & 0 \\
0 & R
\end{array}\right]: R \in \mathbf{SO}(n)\right\} \cong \mathbf{SO}(n)
\end{align*}
where $\mathbf{SO}(n)$ is the group of $n \times n$ orthogonal matrices with determinant 1.
\end{fact}
Hence, the hyperbolic space is a homogeneous Riemannian manifold and can be written as a quotient space, $\mathbb{L}^n = \mathbf{SO}^+(1,n)/\mathbf{SO}(n)$.

\begin{fact}[\cite{moretti2002interplay}]
A Lorentz transformation $A \in \mathbf{SO}^+(1,n)$ can be decomposed using a polar decomposition and expressed as
\begin{equation*}\label{eq:decomposition1}
    A = \left[ \begin{array}{cc}
1 & 0 \\
0 & R
\end{array}\right]
 \left[ \begin{array}{cc}
c & v^T \\
v & \sqrt{I_n+vv^T}
\end{array}\right]
\end{equation*}
where $R \in \mathbf{SO}(n)$, $v \in \mathbb{R}^n$  and $c = \sqrt{\lVert v\rVert+1}$.
\end{fact}
The first component is called a \emph{Lorentz rotation} and the second component is called a \emph{Lorentz boost}. 

\begin{fact}
Every Lorentz transformation matrix $A \in \mathbf{SO}^+(1,n)$ can be decomposed into
\begin{equation}\label{eq:decomposition2}
    A = \left[ \begin{array}{cc}
1 & 0 \\
0 & P
\end{array}\right]
 \left[ \begin{array}{ccc}
\cosh{\alpha} & \sinh{\alpha} & 0^T\\
\sinh{\alpha} & \cosh{\alpha} & 0^T\\
0& 0 &I_{n-1}
\end{array}\right]
\left[ \begin{array}{cc}
1 & 0 \\
0 & Q^T
\end{array}\right]
\end{equation}
where $P,Q \in \mathbf{SO}(n)$,  $\alpha \in \mathbb{R}$ and $0 \in \mathbb{R}^{n-1}$. See Figure~\ref{fig:lorentz_transformation} for examples of the Lorentz rotations and the Lorentz boosts.
\end{fact}

The matrix in the middle is the Lorentz boost along the first coordinate axis. This decomposition will be very useful in the optimization problem stated in Section~\ref{sec:optimization}, equation \eqref{eq:linear layer}.

We now conclude this section by presenting the explicit closed form formulae for the exponential map and the geodesic distance. For any $\boldsymbol{x} \in \mathbb{L}^n$ and $\boldsymbol{v} \in T_p\mathbb{L}^n$ (the tangent space of $\mathbb{L}^n$ at $\boldsymbol{x}$), the exponential map at $\boldsymbol{x}$ is given by 
\begin{equation}\label{eq:expmap}
    \text{Exp}_{\boldsymbol{x}}(\boldsymbol{v}) = \cosh(\lVert \boldsymbol{v}\rVert_L) \boldsymbol{x}+\sinh(\lVert \boldsymbol{v}\rVert_L) \boldsymbol{b}/\lVert \boldsymbol{v}\rVert_L.
\end{equation}
Since $\mathbb{L}^n$ is a negatively curved Riemannian manifold, its exponential map is invertible and the inverse of the exponential map, also called the Log map, is given by
\begin{equation}\label{eq:logmap}
    \text{Log}_{\boldsymbol{x}}(\boldsymbol{y}) = \frac{\theta}{\sinh(\theta)}(\boldsymbol{y}-\cosh(\theta)\boldsymbol{x})
\end{equation}
where $\boldsymbol{x}, \boldsymbol{y} \in \mathbb{L}^n$ and $\theta$ is the geodesic distance between $\boldsymbol{x}$ and $\boldsymbol{y}$ given by
\begin{equation}\label{eq:geodesics}
    \theta = d_{\mathbb{L}}(\boldsymbol{x},\boldsymbol{y}) = \cosh^{-1} (-\langle \boldsymbol{x},\boldsymbol{y}\rangle_L).
\end{equation}

\section{Nested Hyperbolic Spaces and Networks}\label{sec:mainwork}
\begin{figure}[t]
  \centering
   \includegraphics[width=0.9\linewidth, trim=120 450 100 200,clip]{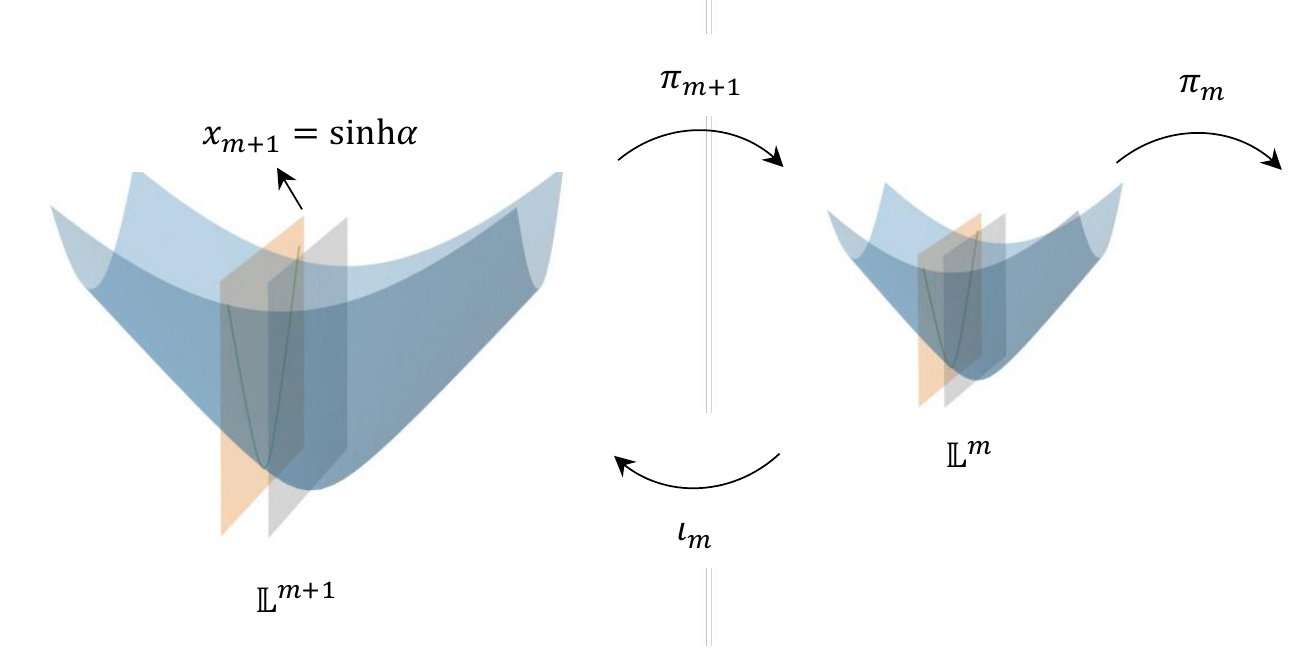}
    \label{fig:nestedhyperboloid}
   \caption{Illustration of NH model using the embedding $\iota_m$ in Eq.~\eqref{eq:embedding} of $\mathbb{L}^{m}$ into $\mathbb{L}^{m+1}$. The $m$-dimensional nested hyperboloid in $\mathbb{L}^{m+1}$ is indeed the intersection of $\mathbb{L}^{m+1}$ and an $m$-dimensional hyperplane.}
   \label{fig:nested_illustraion}
\end{figure}

In this section, we first present the construction of nested hyperboloids (NHs); an illustration of the NHs are given in Figure~\ref{fig:nested_illustraion}. We also prove that the proposed NHs possess several nice properties, including the isometry property and the equivariance under the Lorentz transformations. Then we use the NH representations to design a novel graph convolutional network architecture, called Nested Hyperbolic Graph Convolutional Network (NHGCN).

%We first define the nested structure in the hyperbolic space. The nested hyperbolic representation is analogous to nested sphere \cite{jung2012analysis}.
\subsection{The Nested Hyperboloid Representation}

The key steps to the development of the NHs are the embedding of $\mathbb{L}^m$ into $\mathbb{L}^n$ for $m < n$ and the projection from $\mathbb{L}^n$ to $\mathbb{L}^m$. The principle  
%as suggested by \cite{yang2021nested}, 
is to define an embedding of the corresponding groups of isometries, $\mathbf{SO}^+(1,m)$ and $\mathbf{SO}^+(1,n)$. % \textcolor{red}{WE ARE HURTING OURSELF BY CITING OUR EARLIER WORK SO MANY TIMES. REVIEWER WILL REJECT THE PAPER SAYING NOTHING IS NOVEL?}

%We first define embedding from low dimensional hyperbolic space to high hyperbolic space and then we will have corresponding projection.
First, we consider the embedding $\tilde{\iota}_m: \mathbf{SO}^+(1,m) \to \mathbf{SO}^+(1,m+1)$ defined by 
\begin{align}\label{eq:adapted-GS}
    \tilde{\iota}_m (\boldsymbol{O}) = \text{adapted-GS}\left( \boldsymbol{\Lambda} \left[ \begin{array}{cc}
        \boldsymbol{O} & \boldsymbol{a^T} \\
        \boldsymbol{b} & c
    \end{array} \right] \right)
\end{align}
where $\boldsymbol{O} \in \mathbf{SO}^+(1,m)$, $a,b\in \mathbb{R}^{m+1}$, $c \neq a^TO^{-1}b$, and $\Lambda \in \mathbf{SO}^+(1,m+1).$ The function $\text{adapted-GS}(\cdot)$ is an adaptation of the standard Gram-Schmidt process to orthonormalize vectors with respect to the Lorentz inner product defined earlier.

The Riemannian submersion (see \cite{helgason1979differential} for the definition of a Riemannian submersion) $\pi:\mathbf{SO}^+(1,m) \to \mathbb{L}^m$ is given by $\pi(O) = O_1$ where $O \in \mathbf{SO}^+(1,m)$ and $O_1$ is the first column of $O$. Therefore, the induced embedding $\iota_m: \mathbb{L}^m \rightarrow \mathbb{L}^{m+1}$ is 
\begin{align}
    \iota_m(\boldsymbol{x}) 
= \boldsymbol{\Lambda} \left[ \begin{array}{c}
 \cosh(r) \boldsymbol{x}  \\
\sinh(r)
\end{array} 
\right]
= \cosh(r)\tilde{\boldsymbol{\Lambda}}\boldsymbol{x} + \sinh(r)\boldsymbol{v} \label{eq:embedding}
\end{align}
where $\boldsymbol{\Lambda} = [\tilde{\boldsymbol{\Lambda}}\quad \boldsymbol{v}] \in \mathbf{SO}^+(1,m+1)$. This class of embeddings is quite general as it includes isometric embeddings as special cases.
\begin{proposition}
The embedding $\iota_m: \mathbb{L}^m \to \mathbb{L}^{m+1}$ is isometric when $r = 0$.
\end{proposition}
\begin{proof}
It follows directly from the definitions of the Lorentz transformation and the geodesic distance on $\mathbb{L}^m$.
\end{proof}
Furthermore, the embedding~\eqref{eq:embedding} is equivariant under Lorentz transformations.
\begin{theorem}
The embedding $\iota_m: \mathbb{L}^m \to \mathbb{L}^{m+1}$ is equivariant under Lorentz transformations of $\mathbf{SO}^+(1,m)$, i.e., $\iota_m(Rx) = \Psi_{\Lambda}(\tilde{\iota}_m(R))\iota_m(x)$ where $\Psi_g(h) = ghg^{-1}$.
\end{theorem}

\begin{proof}
For $x \in \mathbb{L}^m$ and $R \in \mathbf{SO}^+(1,m)$,
\begin{align*}\label{eq:equivariant}
    \iota_m(Rx) & = \Lambda\left[ \begin{array}{c} \cosh(r)Rx \\ \sinh(r) \end{array} \right]\\
& = \Lambda \left[ \begin{array}{cc} R & 0 \\ 0 & 1 \end{array} \right]
\left[ \begin{array}{c} \cosh(r)x \\ \sinh(r) \end{array} \right]\\
& = \Lambda \left[ \begin{array}{cc} R & 0 \\ 0 & 1 \end{array} \right]\Lambda^{-1}\Lambda
\left[ \begin{array}{c} \cosh(r)x \\ \sinh(r) \end{array} \right]\\
& = \Psi_{\Lambda}(\tilde{\iota}_m(R))\iota_m(x).
\end{align*}
\end{proof}

The projection $\pi_{m+1}:\mathbb{L}^{m+1} \rightarrow \mathbb{L}^m$ corresponding to $\iota_m$ is given by,
\begin{equation}\label{eq:projection1}
\begin{split}
     \pi_{m+1} (\boldsymbol{x}) &= \frac{1}{\cosh{r}}J_m \tilde{\boldsymbol{\Lambda}}^TJ_{m+1}\boldsymbol{x}\\
     &= \frac{J_{m} \tilde{\boldsymbol{\Lambda}}^TJ_{m+1}\boldsymbol{x}}{\lVert J_{m} \tilde{\boldsymbol{\Lambda}}^TJ_{m+1}\boldsymbol{x} \rVert_L}
\end{split}
\end{equation}
for $\boldsymbol{x} \in \mathbb{L}^{m+1}$. 
Hence, the reconstructed point $\hat{\boldsymbol{x}}\in \mathbb{L}^{m+1}$ of $\boldsymbol{x} \in \mathbb{L}^{m+1}$ is 
\begin{align}\label{eq:reconstructed_point}
    \hat{\boldsymbol{x}} = \cosh (r) \boldsymbol{\Lambda} \frac{J_{m} \tilde{\boldsymbol{\Lambda}}^TJ_{m+1}\boldsymbol{x}}{\lVert J_{m} \tilde{\boldsymbol{\Lambda}}^TJ_{m+1}\boldsymbol{x} \rVert_L} + \sinh(r) \boldsymbol{v}.
\end{align}
The unknowns $\boldsymbol{\Lambda}= [\tilde{\boldsymbol{\Lambda}}\quad \boldsymbol{v}]$ and $r$ can then be obtained by minimizing the reconstruction error
\begin{align}\label{reconstruction_error}
    L(\boldsymbol{\Lambda},r) = \frac{1}{N} \sum\limits_{i=1}^{N} \left(d_{\mathbb{L}}(\boldsymbol{x}_i,\hat{\boldsymbol{x}_i})\right)^2.
\end{align}

The projection of $\boldsymbol{x} \in \mathbb{L}^n$ into $\mathbb{L}^m$ for $n>m$ can be obtained via the composition $\pi:= \pi_{m+1} \circ \dots \circ \pi_{n}$
\begin{equation}\label{eq:projection}
\begin{split}
         \pi(\boldsymbol{x}) &= J_{m} \left(\prod_{i=m+1}^{n} \frac{1}{\cosh(r_i)}  \tilde{\boldsymbol{\Lambda}}_i \right)^T J_{n}\boldsymbol{x}\\
          & = \frac{J_{m} \boldsymbol{M}^T J_{n}\boldsymbol{x}}{\lVert J_{m} \boldsymbol{M}^T J_{n}\boldsymbol{x} \rVert_L}
\end{split}
\end{equation}
where  $\boldsymbol{M} = \prod_{i=m+1}^n \tilde{\boldsymbol{\Lambda}}_i \in \mathbb{R}^{(n+1) \times (m+1)}$.
\subsection{Nested Hyperbolic Graph Convolutional Network (NHGCN)}

The Hyperbolic Graph  Convolutional Network (HGCN) proposed in \cite{chami2019hyperbolic} is a generalization of Euclidean Graph Network to a hyperbolic space. There are three different layers in HGCN: feature transformation, neighborhood aggregation and non-linear activation. We use our NH representation to define a hyperbolic feature transformation, the weighted Fr\'{e}chet mean to define the neighborhood aggregation and use a tangent ReLU activation. This leads to a novel HGCN architecture. Figure \ref{fig:architecture} depicts the HGCN architecture. Each of the three distinct layers are described in detail below. 

%\begin{figure} \centering
%    \begin{subfigure}[b]{0.3\linewidth}
%    \centering
%        \includegraphics[width=0.8\linewidth, trim=200 350 120 100,clip]{figures/network2.pdf}
%        %\caption{Input node and its neighbors}\label{fig:input_nodes}
%    \end{subfigure} 
%    \begin{subfigure}[b]{0.3\linewidth}   
%    \centering
%        \includegraphics[width=0.8\linewidth, trim=200 350 120 100,clip]{figures/network1.pdf}
%        %\caption{Feature transformation}\label{fig:feature_transormation}    
%    \end{subfigure} 
%    \begin{subfigure}[b]{0.3\linewidth}   
%    \centering
%        \includegraphics[width=0.8\linewidth, trim=200 350 120 100,clip]{figures/network3.pdf}
%        %\caption{Neighborhood aggregation}\label{fig:Neibor_aggregation}    
%    \end{subfigure} 
%    \caption{Hyperbolic Graph Network Architecture}
%    \label{fig:architecture}
%\end{figure}

\begin{figure}
    \centering
    \includegraphics[height=8cm]{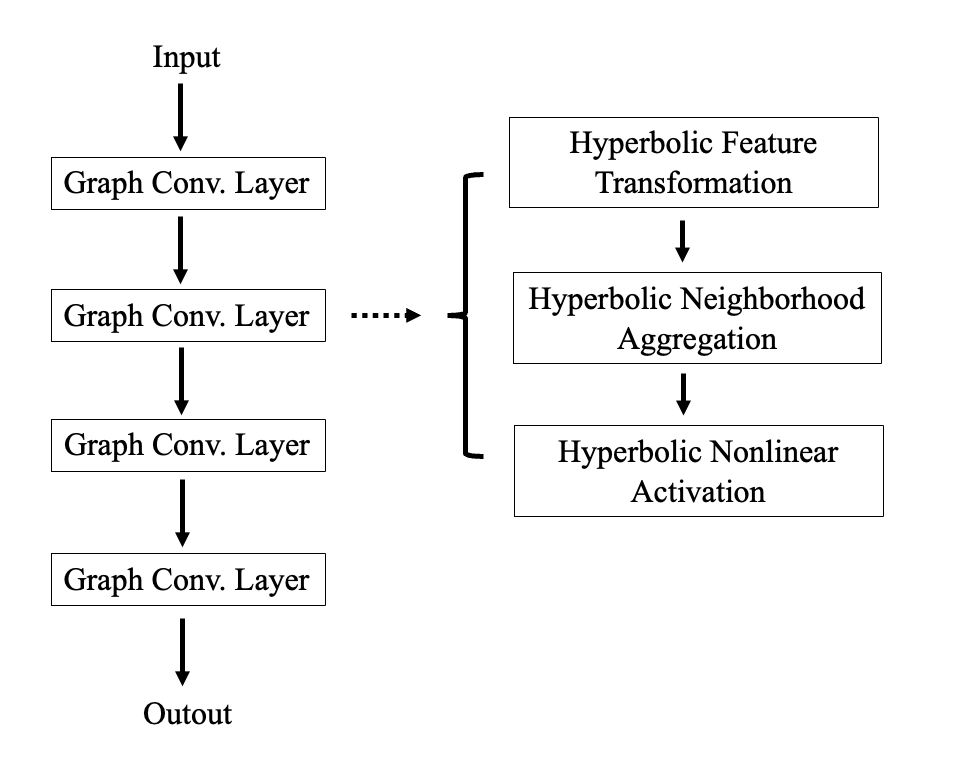}
    \caption{The HGCN Architecture}
    \label{fig:architecture}
\end{figure}

% \textcolor{red}{Please put a small figure of the architecture here. I have already written a short sentence here. Use the caption:HGCN Architecture.}

\textbf{Hyperbolic Feature Transformation:} 
Given $\boldsymbol{x} \in \mathbb{L}^n$, the hyperbolic feature transformation is defined using \eqref{eq:projection} as follows
\begin{equation}\label{eq:linear layer}
\begin{split}
    \boldsymbol{y} = \frac{\boldsymbol{W} \boldsymbol{x}}{\lVert \boldsymbol{W} \boldsymbol{x} \rVert_L} \qquad
    \text{s.t. } \boldsymbol{W} J_{n}\boldsymbol{W}^T = J_{m}
\end{split}
\end{equation}
where $\boldsymbol{W} \in \mathbb{R}^{(m+1)\times (n+1)}$. It is easy to prove that $\boldsymbol{y} \in \mathbb{L}^m$.

At the $l$-th layer, the inputs are the hyperbolic representation $\boldsymbol{x}_i^{l-1}$ from the previous layer and the feature transformation matrix is $\boldsymbol{W}^l$. The intermediate hyperbolic representation of $i$-th node is computed as follows
\begin{align}
    \boldsymbol{x}_i^l = \frac{\boldsymbol{W}^l \boldsymbol{x}_i^{l-1}}{\lVert\boldsymbol{W}^l \boldsymbol{x}_i^{l-1} \rVert } \qquad
    \text{s.t. } {\boldsymbol{W}^{l}} J_{n_{l-1}}{{\boldsymbol{W}^{l}}^T} = J_{n_{l}}
\end{align}

\textbf{Hyperbolic Neighborhood Aggregation:}
In GCNs, the neighborhood aggregation is used to combine neighboring features by computing the weighted centroid of these features. The weighted centroid in hyperbolic space of a point set  $\{\boldsymbol{x}_i\}_{i=1} \in \mathbb{L}^n$ is obtained using the weighted Fr\'{e}chet mean. However, it does not have closed form expression in hyperbolic space. We use hyperbolic neighborhood aggregation proposed in \cite{chen2021fully,zhang2021lorentzian}, where aggregated representation for a node $\boldsymbol{x}_i^l$ at $l$-th layer is the weighted centroid $\boldsymbol{\mu}^l_i$ of its neighboring nodes $\{\boldsymbol{x}_j^l\}_{j=1}^p \in  \mathbb{L}^{n_l}$ w.r.t \emph{squared Lorentzian distance}, namely
%is the minimizer of expectation of squared Lorentzian distance, namely
\begin{align}
    \boldsymbol{\mu}_i^l= \arg\min_{\boldsymbol{\mu}^l\in\mathbb{L}^{n_l}} \sum_{j=1}^{p} \nu_j^l d_{\mathbb{L}}^2(\boldsymbol{x}_j^l,\boldsymbol{\mu}_i^l)
\end{align}
where $\nu_j^l$ is the weight for $\boldsymbol{x}_j^l$ and $d^2_{\mathbb{L}} (\boldsymbol{x},\boldsymbol{y}) = -1- \langle\boldsymbol{x},\boldsymbol{y}\rangle_L$ is the squared Lorentzian distance\cite{ratcliffe2006foundations}. Authors in \cite{law2019lorentzian} proved that this problem has closed form solution given by,
 \begin{align}
     \boldsymbol{\mu}_i^l = \frac{\sum_{j=1}^p\nu_j^l\boldsymbol{x}_j^l}{\lvert \lVert \sum_{j=1}^p\nu_j^l\boldsymbol{x}_j^l \rVert_L\rvert}.
 \end{align}

\textbf{Hyperbolic Nonlinear Activation:} A nonlinear activation is required in our network since the feature transform is a linear operation. We choose to apply tangent ReLU to prevents our multi-layer network from collapsing into a single layer network. The tangent ReLU in the hyperbolic space is defined as,
\begin{equation}\label{eq:activation}
\sigma(\boldsymbol{x}_i^l) = \text{Exp}_{\boldsymbol{0}}(\text{ReLU}(\text{Log}_{\boldsymbol{0}}(\boldsymbol{x}_i^l))).
\end{equation}
Here $\boldsymbol{0}= [1,0,\dots,0]^T \in \mathbb{L}^{n_l}$ (correspond to the origin in the Poincar\'{e} model) is chosen as the base point to define the anchor point in the tangent ReLU.

\subsection{Optimization}\label{sec:optimization}
In this section, we will explain how to update parameters in network, i.e.\ transformation matrix $\boldsymbol{W}$ in  \eqref{eq:linear layer}. Instead of updating $\boldsymbol{W}$ directly, we find an alternative way by decomposing $\boldsymbol{W}$ into three matrices using \eqref{eq:decomposition2}. More specifically, we write 
\begin{align*}
    \boldsymbol{W} = \left[ \begin{array}{cc}
1 & 0 \\
0 & \widetilde{\boldsymbol{P}}
\end{array}\right]
 \left[ \begin{array}{ccc}
\cosh{\alpha} & \sinh{\alpha} & 0^T\\
\sinh{\alpha} & \cosh{\alpha} & 0^T\\
0& 0 & I_{n-1}
\end{array}\right]
\left[ \begin{array}{cc}
1 & 0 \\
0 & \boldsymbol{Q}^T
\end{array}\right]
\end{align*}
where $\boldsymbol{Q} \in \mathbf{SO}(n),\alpha \in \mathbb{R}$ and $\widetilde{\boldsymbol{P}}$ is the first $m$ rows of a $\boldsymbol{P} \in \mathbf{SO}(n)$ which is from a \emph{Stiefel manifold}\cite{edelman1998geometry}. Then we regard our feature transformation as a sequence of multiplication by three matrices and update them one by one. 
\section{Experiments}\label{sec:experiment}

In this section, we will first evaluate NH as a dimensionality reduction method compared with HoroPCA, tangent PCA and EPGA. We show that the proposed NH outperforms all of these method on both synthetic data and real data in terms of reconstruction error. Then, we apply the proposed NHGCN to the problems of link prediction and node classification on four graph data sets described in \cite{chami2019hyperbolic}. Our method yields results that are better or comparable to existing hyperbolic graph networks. The implementations are based on Pymanopt\cite{koep2016pymanopt} and GeoTorch\cite{lezcano2019trivializations} for dimensionality reduction and NHGCN respectively.

\subsection{Dimensionality Reduction in Hyperbolic Space}
First we present synthetic data experiments followed by experiments on real data.
\paragraph{Synthetic Experiments} As a dimensionality reduction method, we compare NH with three other competing methods: tangent PCA, EPCA, and HoroPCA. Note that the first two are applicable on any Riemannian manifolds and HoroPCA is proposed specifically for hyperbolic spaces as is our NH space method. The major difference between NH and the aforementioned competitors is that NH does not require the fitted submanifold to pass through the Fr\'{e}chet mean whereas the others do. This extra requirement can sometimes lead to failure in capturing the data trend as shown in Figure~\ref{fig:synthetic_hyperbola}. 

% \begin{figure} \centering
%     \begin{subfigure}[b]{0.4\linewidth}
%         \includegraphics[width=0.8\linewidth]{figures/hyperbola_1.png}
%         \label{fig:hyperbola1}
%     \end{subfigure} 
%     \begin{subfigure}[b]{0.4\linewidth}    
%         \includegraphics[width=0.8\linewidth]{figures/hyperbola_1_Rx08.png}
%         \label{fig:hyperbola1Rx}    
%     \end{subfigure} 
%     \begin{subfigure}[b]{0.4\linewidth}
%         \includegraphics[width=0.8\linewidth]{figures/hyperbola_1_Ry05.png}
%         \label{fig:hyperbola1Ry}
%     \end{subfigure} 
%     \begin{subfigure}[b]{0.4\linewidth}    
%         \includegraphics[width=0.8\linewidth]{figures/hyperbola_1_Rz05.png}
%         \label{fig:hyperbola1Rz}    
%     \end{subfigure} 
%     \caption{Synthetic data in hyperbolic space visualized using the Poincar\'{e} disk model along with principal geodesics obtained using tangent PCA and NH. As evident, NH is better at capturing the trend of the data.}
%     \label{fig:synthetic_hyperbola}
% \end{figure}

\begin{figure} \centering
    \begin{subfigure}[b]{0.4\linewidth}
    \centering
        \includegraphics[width=0.8\linewidth]{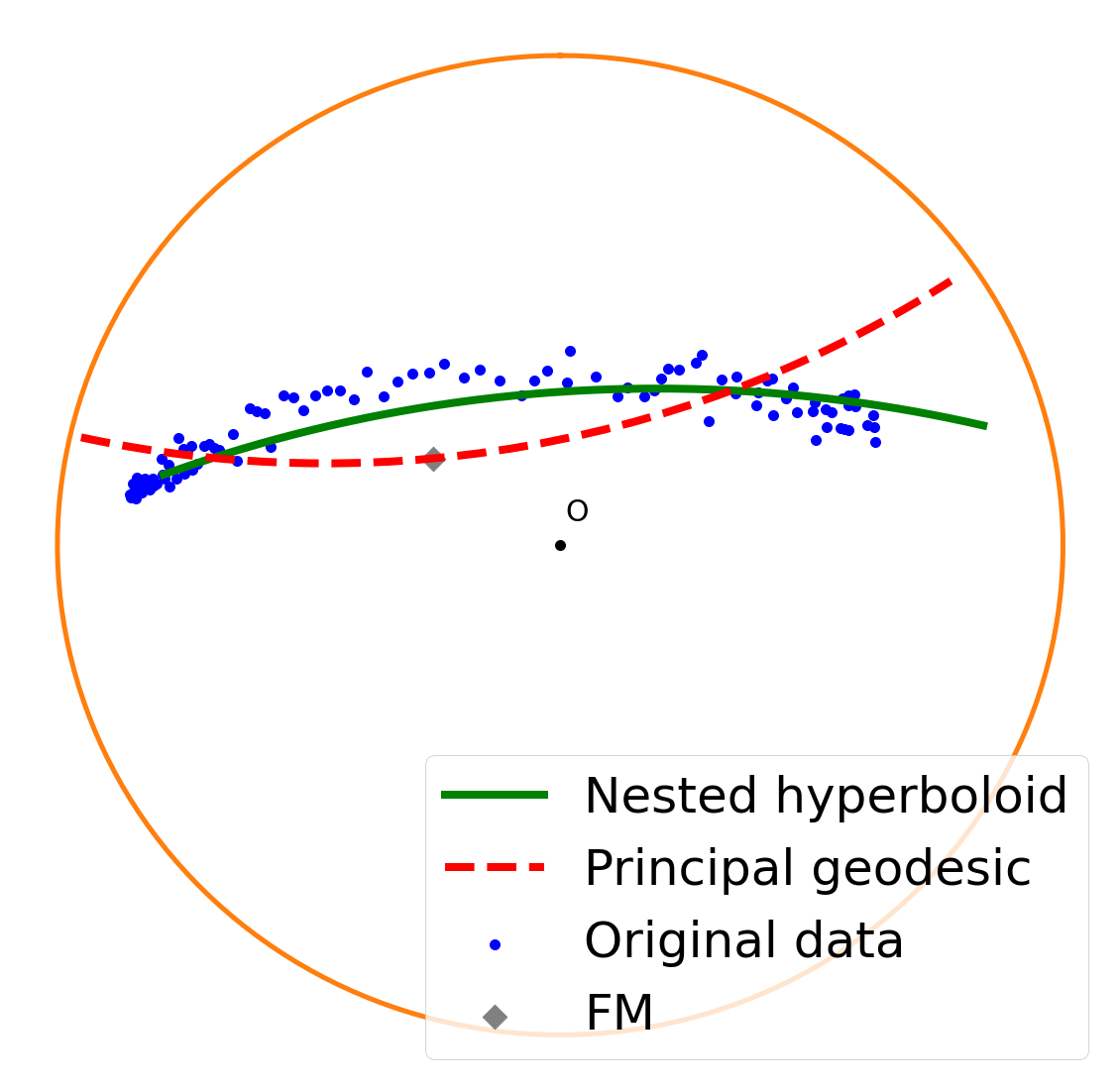}
        \label{fig:hyperbola1}
    \end{subfigure} 
    \begin{subfigure}[b]{0.4\linewidth}
    \centering
        \includegraphics[width=0.8\linewidth]{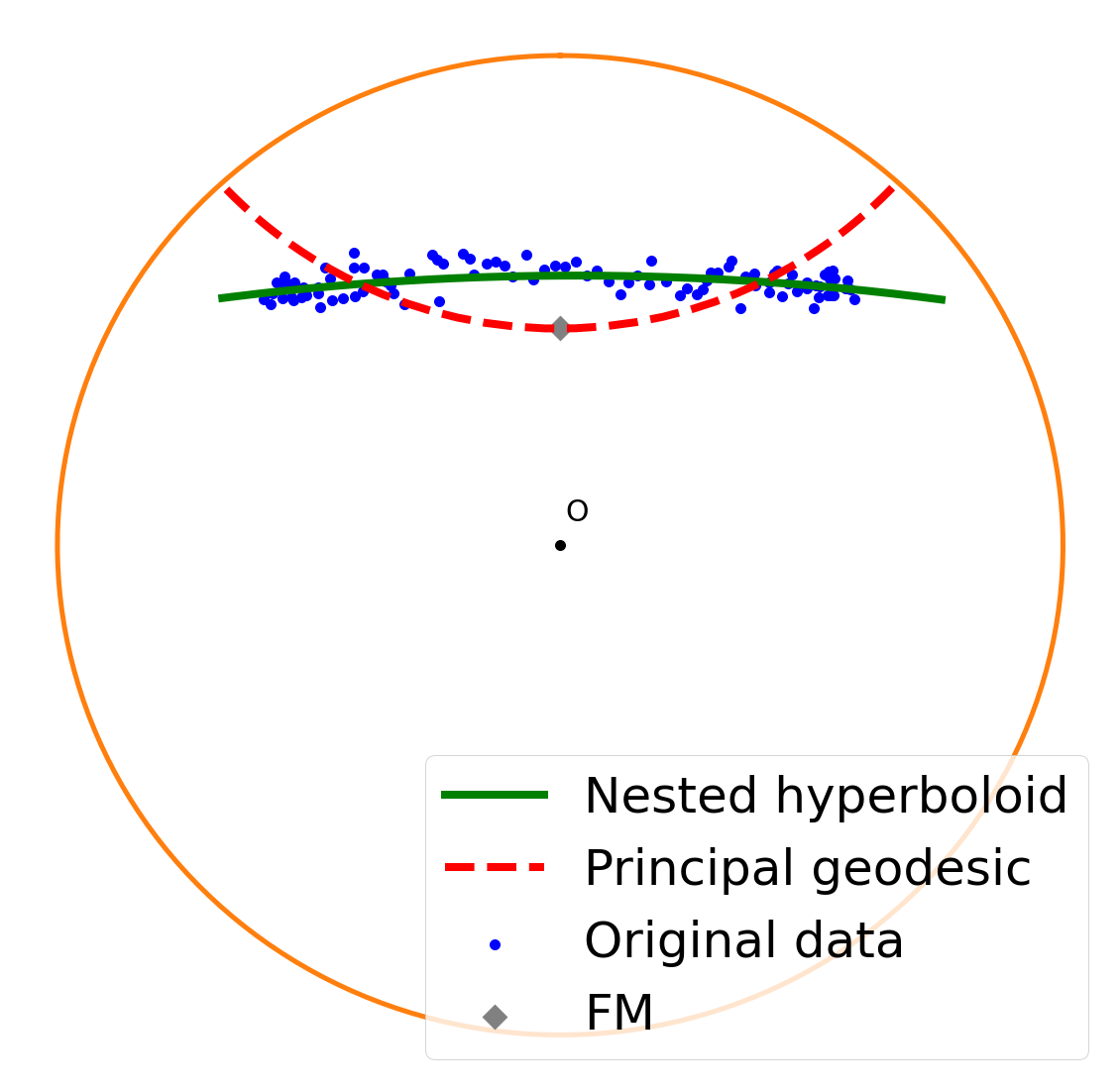}
        \label{fig:hyperbola1Rx}    
    \end{subfigure} 
    \caption{Synthetic data in hyperbolic space visualized using a Poincar\'{e} disk model along with principal geodesic obtained using tangent PCA and the NH. NH is better at capturing the trend of the data since it is not restricted to pass through the Fr\'{e}chet mean.}
    \label{fig:synthetic_hyperbola}
\end{figure}

Apart from visual inspection, we use the reconstruction error as a measure of the goodness of fit. To see how NH performs in comparison to others under different levels of noise, we generate synthetic data from the \emph{wrapped normal distribution} \cite{mathieu2019charline} on $\mathbb{L}^{10}$ with variance ranging from~0.2 to~2. Then we apply different dimensionality reduction methods to reduce the dimension down to 2. The result is shown in Figure~\ref{fig:variance_change}. The results of EPGA and NH are essentially the same. This is due to the fact that the wrapped normal distribution we chose is isotropic around the mean and hence in this case the assumption of submanifold passing through the Fr\'{e}chet mean is valid. Even in this case, we observe a significant improvement of NH over tangent PCA and HoroPCA especially in the large variance scenario. The main reasons are that (i) tangent PCA uses local linearization which would lead to inaccuracies when the data is not tightly clustered around the Fr\'{e}chet mean and (ii) the HoroPCA seeks to maximize the projected variance on the submanifold, which, as is well known, not equivalent to minimizing the reconstruction error. There is a clear justification for the choice of using reconstruction error as the objective function since, \emph{we want a good approximation of the original data with the lower-dimensional representation}.

\begin{figure}[t]
  \centering
   \includegraphics[height=6cm]{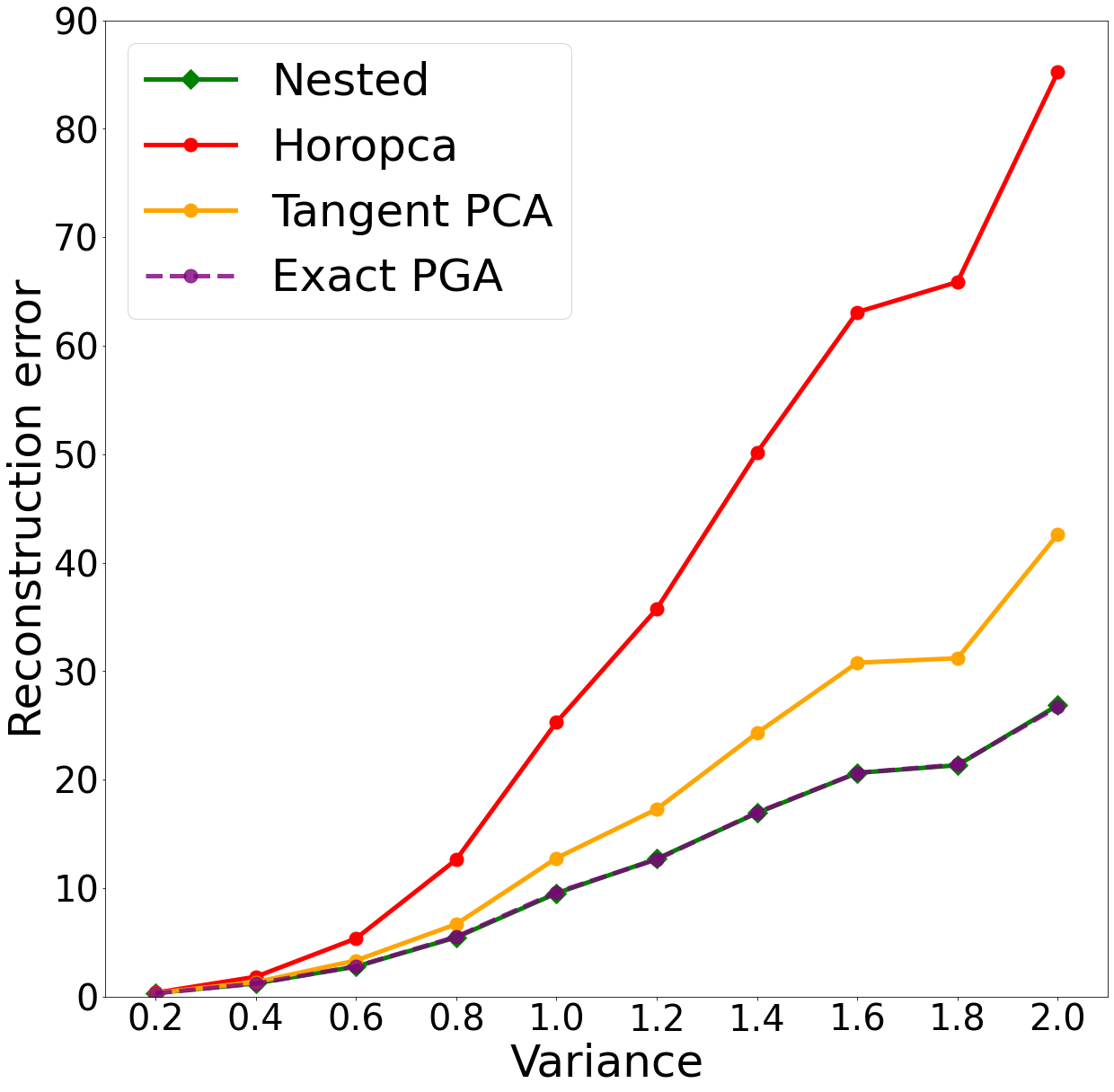}
   \caption{Reconstruction errors for $\mathbb{L}^{10}$ to $\mathbb{L}^2$. The data is generated from \emph{wrapped normal distributions} \cite{mathieu2019charline} with variances ranging from 0.2 to 2.}
   \label{fig:variance_change}
\end{figure}

%We compare NH to 3 dimentionality reduction methods: tangent PCA, EPGA and HoroPCA.

%We test our method on 4 datasets as in HoroPCA\cite{chami2021horopca} and construct another two unbalanced tree dataset by cutting edges in balanced tree dataset. 

%We report the reconstruction error in Table \ref{tab:dim reduction}, i.e.
%\begin{align*}
%    \frac{1}{N} \sum\limits_{i=1}^{N} d_{\mathbb{H}}^2(\boldsymbol{x}_i,\hat{\boldsymbol{x}_i})
%\end{align*}

\paragraph{Hyperbolic Embeddings of Trees} 

\begin{table*}
\centering
\scalebox{0.85}{
\begin{tabular}{|l|c|c|c|c|c|c|}
\toprule
Datasets & balancedtree & unbalanced1 & unbalanced2   & phylo tree &diseasome &ca-CSphd \\ \midrule
tPCA & 5.75  & 4.98 & 4.86 & 121.19& 21.53 & 71.67\\
HoroPCA & 7.80$\pm$0.06  &  6.51$\pm$0.28 & 7.35$\pm$0.61 & 108.62$\pm$9.20& 26.94$\pm$0.99 & 87.99$\pm$4.69\\
EPGA & \underline{4.01}$\pm$0.76  &   \underline{3.23}$\pm$0.08 & \underline{3.33}$\pm$0.46 &
\underline{25.93}$\pm$0.99 & \underline{9.72}$\pm$0.36 & 
\underline{22.98}$\pm$0.23\\
Nested & \textbf{3.35}$\pm$0.05 & \textbf{3.10}$\pm$0.01 & 
\textbf{3.22}$\pm$0.06 & 
\textbf{24.11}$\pm$0.68&
\textbf{9.18}$\pm$0.10  & 
\textbf{22.68}$\pm$0.40\\ \bottomrule
\end{tabular}
}
\caption{Reconstruction errors from $\mathbb{L}^{10}$ to $\mathbb{L}^2$. The numbers depicted are: mean error $\pm$ standard dev.\ of error. Numbers in bold indicate the method with the smallest errors while underlined numbers indicate the second best results. }
\label{tab:dim_reduction}
\end{table*}

\begin{table*}
\centering
\scalebox{0.85}{
\begin{tabular}{|l|c|c|c|c|c|c|c|c|}
\hline
 & \multicolumn{2}{c}{Disease} & \multicolumn{2}{c}{Airport} & \multicolumn{2}{c}{PubMed} & \multicolumn{2}{c|}{Cora} \\ \cline{2-9} 
Task & LP & NC & LP & NC & LP & NC & LP & NC \\ \hline
GCN\cite{kipf2016semi} & 64.7$\pm$0.5 & 69.7$\pm$0.4 & 89.3$\pm$0.4 & 81.4$\pm$0.6 & 91.1$\pm$0.5 & 78.1$\pm$0.2 & 90.4$\pm$0.2 & 81.3$\pm$0.3 \\
GAT\cite{velivckovic2018graph} & 69.8$\pm$0.3 & 70.4$\pm$0.4 & 90.5$\pm$0.3 & 81.5$\pm$0.3 & 91.2$\pm$0.1 & 79.0$\pm$0.3 & \underline{93.7}$\pm$0.1 & \underline{83.0}$\pm$0.7 \\
SAGE\cite{hamilton2017inductive} & 65.9$\pm$0.3 & 69.1$\pm$0.6 & 90.4$\pm$0.5 & 82.1$\pm$0.5 & 86.2$\pm$1.0 & 77.4$\pm$2.2 & 85.5$\pm$0.6 & 77.9$\pm$2.4 \\
SGC\cite{wu2019simplifying} & 65.1$\pm$0.2 & 69.5$\pm$0.2 & 89.8$\pm$0.3 & 80.6$\pm$0.1 & 94.1$\pm$0.0 & 78.9$\pm$0.0 & 91.5$\pm$0.1 & 81.0$\pm$0.1 \\ \hline
HGCN\cite{chami2019hyperbolic} & 90.8$\pm$0.3 & 74.5$\pm$0.9 & 96.4$\pm$0.1 & 90.6$\pm$0.2 & 96.3$\pm$0.0 & \underline{80.3}$\pm$0.3 & 92.9$\pm$0.1 & 79.9$\pm$0.2 \\
H2H-GCN\cite{dai2021hyperbolic} & \textbf{97.0}$\pm$0.3 & 88.6$\pm$1.7 & 96.4$\pm$0.1 & 89.3$\pm$0.5 &  \textbf{96.9}$\pm$0.0 & 79.9$\pm$0.5& \textbf{95.0}$\pm$0.0 & 82.8$\pm$0.4 \\
HYBONET\cite{chen2021fully} & 96.3$\pm$0.3 &  \textbf{94.5}$\pm$0.8 & \underline{97.0}$\pm$0.2 &  \textbf{92.5}$\pm$0.9 & 96.4$\pm$0.1 & 77.9$\pm$1.0 & 94.3$\pm$0.3 & 81.3$\pm$0.9 \\
LGCN\cite{zhang2021lorentzian} & \underline{96.6}$\pm$0.6 & 84.4$\pm$0.8 & - & - & \underline{96.6}$\pm$0.1 & 78.6$\pm$0.7 & 93.6$\pm$0.4 & \textbf{83.3}$\pm$0.7 \\
NHGCN(Ours) & 92.8$\pm$0.2 &\underline{91.7}$\pm$0.7&\textbf{97.2}$\pm$0.3& \underline{92.4}$\pm$0.7 &\textbf{96.9}$\pm$0.1 & \textbf{80.5}$\pm$0.0& 93.6$\pm$0.2 & 80.3$\pm0.8$\\ \hline
\end{tabular}
}
\caption{Area under the ROC test results (\%) for link prediction (LP), and F1 scores (\%) for node classification (NC). The results of other networks are obtained from the original papers and in \cite{zhang2021lorentzian}, the authors did not test their network on the Airport dataset.}
\label{tab:Graph_nn}
\end{table*}

For real data experiments, we consider reducing the dimensionality of trees that are embedded into a hyperbolic space. We validate our method on the four datasets described in \cite{sala2018representation} including (i) a fully balanced tree, (ii) a phylogenetic tree, (iii) a biological graph comprising of diseases’ relationships, and (iv) a graph of Computer Science (CS) Ph.D.\ advisor-advisee relationships. We also create another two datasets by removing some edges in the balanced tree dataset. We apply the method in \cite{gu2018learning} to embed the tree datasets into a Poincar\'{e} ball of dimension 10 and then apply our NH along with other competing dimensionality reduction methods to reduce the dimension down to 2. The results are reported in Table~\ref{tab:dim_reduction}. In Table~\ref{tab:dim_reduction}, we report the means and the standard deviations of the reconstruction errors for EPGA, HoroPCA and NH. From the table, we can see that our method performs the best among other methods. Especially, the HoroPCA is worse than the tangent PCA and EPGA in terms of reconstruction error, through it shows higher explained variance in \cite{chami2021horopca}. The reason might be that HoroPCA seeks projections that maximize the explained variance which is not equivalent to minimizing the reconstruction error in the Riemannian manifold case.

\subsection{Nested Hyperbolic Graph Networks} 

To evaluate the power of the proposed NHGCN, we apply it to problems of link prediction and node classification. We use four public domain datasets: Disease \cite{chami2019hyperbolic}, Airport \cite{chami2019hyperbolic}, PubMed \cite{namata2012query}, and Cora \cite{sen2008collective}. We compare our NHGCN with many other graph neural networks and the results are reported in Table~\ref{tab:Graph_nn}. For the link prediction (LP), we report the means and the standard deviation of the area under the receiver operating characterization (ROC) curve on the test data; for the problem of node classification (NC), we report the mean and the standard deviation of the F1 scores. As evident from the table, our results are comparable to the state-of-the-art and in three cases better. Our reported results can be improved further via the use of better Riemannian ADAM used in this work e.g., one with a built in variance reduction \cite{sato2019riemannian}.

\section{Conclusion}\label{sec:conclusion}
%Hyperbolic spaces have attracted enormous attention in the recent past in the machine learning community. This is primarily due to their geometry being well suited for hierarchically organized data sets. Many real world data sets exhibit heirarchical structure and thus can benefit from embedding in a hyperbolic space.
In this paper, we presented a novel dimensionality reduction technique in hyperbolic spaces called the nested hyperbolic (NH) space representation. NH representation was constructed using a projection operator that was shown to yield isometric embeddings and further was shown to be equivariant to the isometry group admitted by the hyperbolic space. Further, we empirically showed that it yields lower reconstruction error compared to the state-of-the-art (HorroPCA, PGA, tPCA). Using the NH representation, we developed a novel fully HGCN and tested it on several data sets. Our NHGCN was shown to achieve comparable to superior performance over several competing methods.

\vspace{1cm}
\textbf{Acknowledgement: } This research was in part funded by the NSF grant IIS-1724174 to Vemuri. 

% \include{hyperboloidmodel}

%%%%%%%%% REFERENCES
{\small
\bibliographystyle{ieee_fullname.bst}
\bibliography{reference.bib}
}

\end{document}